\documentclass[twoside,11pt]{article}

\usepackage{xcolor}
\usepackage{jmlr2e}
\usepackage{algorithm}
\usepackage{algorithmic}
\usepackage{amsmath}
\usepackage{tikz}
\usetikzlibrary{positioning}

\newcommand{\commentout}[1]{}
\newcommand{\eat}[1]{}

\newcommand{\calU}{{\mathcal U}}

\newcommand{\tran}{\mathrm{W}}

\definecolor{olive}{rgb}{0.3, 0.4, .1}
\definecolor{fore}{RGB}{249,242,215}
\definecolor{back}{RGB}{51,51,51}
\definecolor{title}{RGB}{255,0,90}
\definecolor{dgreen}{rgb}{0.,0.6,0.}
\definecolor{gold}{rgb}{1.,0.84,0.}
\definecolor{JungleGreen}{cmyk}{0.99,0,0.52,0}
\definecolor{BlueGreen}{cmyk}{0.85,0,0.33,0}
\definecolor{RawSienna}{cmyk}{0,0.72,1,0.45}
\definecolor{Magenta}{cmyk}{0,1,0,0}

\newcommand{\veps}{\varepsilon}

\jmlrheading{1}{2017}{1-48}{10/17}{}{meila00a}{Shichuan Deng, Wenzheng Li and Xuan Wu}

\ShortHeadings{Wasserstein Identity Testing}{Deng, Li and Wu}
\firstpageno{1}

\begin{document}

\title{Wasserstein Identity Testing}

\author{\name Shichuan Deng \email  dsc15@mails.tsinghua.edu.cn\\
       \addr Institute for Interdisciplinary Information Sciences\\
       Tsinghua University\\
       Beijing, China
       \AND
       \name Wenzheng Li \email  wz-li14@mails.tsinghua.edu.cn\\
       \addr Institute for Interdisciplinary Information Sciences\\
       Tsinghua University\\
       Beijing, China
       \AND
       \name Xuan Wu \email wuxuan15@mails.tsinghua.edu.cn \\
       \addr Institute for Interdisciplinary Information Sciences\\
       Tsinghua University\\
       Beijing, China}

\editor{} 

\maketitle

\begin{abstract}
Uniformity testing and the more general identity testing
are well studied problems in distributional property testing.
Most previous work focuses on testing under $L_1$-distance.
However, when the support is very large or even continuous,
testing under $L_1$-distance may require a huge (even infinite) number of samples.
Motivated by such issues, we consider the identity testing in
Wasserstein distance (a.k.a. transportation distance and earthmover distance) on a metric space (discrete or continuous).

In this paper, we propose the Wasserstein identity testing problem (Identity Testing in Wasserstein distance).
We obtain nearly optimal worst-case sample complexity for the problem.
Moreover, for a large class of probability distributions satisfying
the so-called "Doubling Condition", we provide nearly instance-optimal sample complexity.
\end{abstract}

\begin{keywords}
Identity Testing, Uniformity Testing, Distribution Property Testing, Property Testing, Wasserstein Distance.
\end{keywords}

\section{Introduction}

Property Testing is proposed in the seminal work of Goldreich et.al (\cite{OSD98}), which is generally the study of designing and analyzing of randomized decision algorithm on efficiently making decision whether the given instance is having certain property or somewhat far from having it. Significantly, the query complexity of efficient property testing algorithm is often sublinear on the size of its accessing instance.

In recent years, distribution property testing has received much attention from theoretical computer science research. On most problems in distribution property testing, the input is a set of independent samples from an unknown distribution, and the decision is on whether the distribution has certain properties or not. Researchers have investigated the sample complexity of various testing problems of distribution properties such as uniformity, identity to certain distribution, closeness testing, having small entropy, having small support, being uniform on a small subset and so on (\cite{GR00},\cite{Paninski08}, \cite{CDVV14}, \cite{VV14},\cite{DBLP:journals/corr/DiakonikolasKN14}, \cite{unseenVV},\cite{cltVV}, \cite{BC17}, \cite{2017arXiv170902087D},\cite{iliasnew16}).

In this paper, we focus on the problem of identity testing. Arguably, identity testing together with its special case uniformity testing are the best studied problems in distribution property testing. In identity testing, we are given sample access to an unknown distribution $q$, the explicit description of a known distribution $p$ and a proximity parameter $\varepsilon>0$. Then we are required to distinguish the following two cases: 1) $q$ is identical to $p$; 2) certain distance (e.g. the $L_1$-distance, the Hellinger distance, or the $k$th Wasserstein distance) between $p$ and $q$ is larger than $\varepsilon$.

The sample complexity of identity testing in $L_1$-distance (equivalently statistical distance or total variation distance) is now fully understood in a series of work (\cite{GR00}, \cite{Paninski08}, \cite{CDVV14}). Specifically, testing if a distribution supported on $[n]$ is uniform with proximity parameter $\varepsilon>0$ in $L_1$-distance requires $\Theta(\sqrt{n}/\varepsilon^2)$ many samples (\cite{Paninski08},\cite{VV14}).
However, consider the case where the support is continuous, the bound above
becomes meaningless. For example, the natural problem of
testing whether a distribution supported on $[0,1]$ is uniform in $L_1$-distance would require an infinite number of samples.

Motivated by the these issues, we would like to study the testing problem
under a probability distance that metrizes the weak convergence
(on the other hand, convergence in $L_1$-distance is a strong convergence).
A popular choice is the Wasserstein distance (a.k.a. transportation distance or earthmover distance, see Definition \ref{deftran}). Using Wasserstein distance, identity testing is well defined in arbitrary, even continuous, metric space (with Borel point sets of positive finite measure). We also note that using
Wasserstein distance as the defining metric has gained significant attention
in machine learning community recently (e.g., in generative models \cite{arjovsky2017wasserstein} and mixture models \cite{li2015learning} ).

\begin{definition}[Wasserstein Distance]
\label{deftran}
Let $p,q$ be two distributions supported on metric space $(X,d)$, the Wasserstein distance (or transportation distance) between $p$ and $q$ with respect to $d$ is defined to be:
$$
\tran_d(p,q)=\inf_{M\in \mathrm{coup}(p,q)} \int d(x,y) dM(x,y)
$$
where $\mathrm{coup}(p,q)$ is the set of all coupling distributions of $p$ and $q$, i.e. all distributions on $X\times X$ that have marginal distributions $p$ and $q$.

\end{definition}

We define the problem of Wasserstein identity testing as the following.

\begin{definition}[Wasserstein Identity Testing, $\mathrm{WIT}$] \label{pro}
\label{probdef}
Let $(X,d)$ be a metric space and $p$ a distribution on $X$. For a proximity parameter $\varepsilon>0$, denote $\mathrm{WIT}(X,d,p,\varepsilon)$ the problem of designing an algorithm which, given sample access to an unknown distribution $q$,
\begin{itemize}
\item accepts with probability at least $2/3$ if $p=q$;\\
\item rejects with probability at least $2/3$ if $\tran_d(p,q)\geq \varepsilon$.
\end{itemize}
Moreover, investigating the sample complexity lower bound of these algorithms.

For notational convenience, we use $\mathrm{WIT}(p,\varepsilon)$ for short when there is no risk of ambiguity. Denote by $\calU$ the uniform distribution on $X$, then $\mathrm{WIT}(\calU,\varepsilon)$  is the Wasserstein uniformity testing problem.
\end{definition}

\begin{remark}
When $X$ is not discrete, the meaning of "the explicit description of $p$" is confusing.  However, whenever $X$ is separable, we can estimate $p$ by a distribution $p^*$ supported on a countable $\varepsilon$-net of $X$ for any $\varepsilon>0$. This transformation makes the problem discrete. To make life even easier, in what follows we assume that $|X|<\infty$ all the time and the explicit description of $p$ is well defined.
\end{remark}

\paragraph{Testing versus Learning} A direct approach of identity testing is to learn the distribution. Specifically, on testing if an unknown distribution $q$ is uniform with proximity parameter $\varepsilon>0$, we can estimate the unknown distribution $q$ by the empirical distribution $\hat{q}$ such that the distance between $q$ and $\hat{q}$ is less than $\varepsilon/10$. Then we accept if the distance between $\hat{q}$ and the uniform distribution is less than $\varepsilon/10$ and reject otherwise. A tester is efficient if it uses less samples than estimating by empirical distribution. For statistical efficiency, we are seeking for such efficient tester. For example, if the support is $[n]=\{1,2,...,n\}$ and the distance is $L_1$ distance, the sample complexity of learning is $\Theta(n^2/\varepsilon^2)$ (see e.g. \cite{devroye2001combinatorial}) while the sample complexity of testing is $\Theta(\sqrt{n}/\varepsilon^2)$ (\cite{Paninski08},\cite{VV14}). 

In our case for Wasserstein identity testing, in the natural metric space $d$-dimension hypercube $[0,1]^d$ for $d>3$ equipped with the Euclidean metric, the Wasserstein Law of Large Number (see e.g. \cite{van2014probability}) shows that $\Theta(\varepsilon^{-d})$ many samples are sufficient and necessary to estimate the distribution up to $\varepsilon$ in Wasserstein distance. Hence we automatically obtain a tester with sample complexity $O(\varepsilon^{-d})$ for the problem $\mathrm{WIT}([0,1]^d,L_2,\calU,\varepsilon),d>3$. On the other hand, in Corollary \ref{corohyp}, a tester with sample complexity $\tilde{\Theta}(\varepsilon^{-d/2})$ for the problem $\mathrm{WIT}([0,1]^d,L_2,\calU,\varepsilon),d>3$ is given.

\paragraph{The Chaining Method} The primary technique in this paper is choosing a sequence of $\varepsilon$-nets then decomposing the original testing problem into multiple easier sub-problems according to the nets. This technique is highly related to Talagrand's "Chaining Method" which plays a central roll on proving upper and lower bounds of stochastic process (\cite{talagrand14}).

\subsection{Main Contributions}

Our first contribution is characterizing the worst-case sample complexity of $\mathrm{WIT}$ in arbitrary metric space by giving nearly optimal upper bound and matching lower bound.

\begin{theorem}\label{informalone}
Let $(X,d)$ be a metric space endowed with a distribution $p$. Let $D$ be its diameter. Let $\{N_i,\log \frac{\varepsilon}{8}\leq i\leq \log D\}$ be a sequence of well-separated $2^i$-net of $(X,d)$ (see Definition \ref{net}). There is an algorithm, given sample access to an unknown distribution $q$ over $X$ and a proximity parameter $\varepsilon>0$,
 \begin{itemize}
 \item accepts with probability at least $2/3$ if $p=q$;
 \item rejects with probability at least $2/3$ if $\tran_{d}(p,q)\geq \varepsilon$.
 \end{itemize} The sample complexity of this algorithm is
 $$\tilde{O}\left(\max\left\{\frac{2^{2i}|N_i|^{1/2}}{\varepsilon^2}:\log \frac{\varepsilon}{8}\leq i\leq \log D\right\}\right). $$
 Moreover, any algorithm which distinguishes the two cases for any fixed $p$ and unknown $q$ with probability at least $2/3$ takes
 $${\Omega}\left(\max\left\{\frac{2^{2i}|N_i|^{1/2}}{\varepsilon^2}:\log \frac{\varepsilon}{8}\leq i\leq \log D\right\}\right)
 $$ many samples in the worst case.
\end{theorem}

Actually, Theorem \ref{informalone} is a worst-case result for problems in Definition~\ref{pro}. The sample complexity bound is oblivious on $p$, the target distribution. One may wonder if we can obtain some instance bound which is nearly optimal for every $p$, like what appeared in \cite{VV14}. We show that if the distribution is not too singular (e.g. highly concentrated on one point), characterized by satisfying the following "Doubling Condition" (see Definition \ref{double}), then we can obtain nearly instance-optimal sample complexity bounds (see Theorem \ref{instance}).

  \begin{definition}[Doubling Condition] \label{double}
 Let $(X,d)$ be a metric space and $p$ be a distribution on $X$. For $x\in X$, $r>0$, denote the ball $B(x,r):=\{y\in X: d(x,y)\leq r\}$. $(X,d,p)$ is said to satisfy the "doubling condition" if there exists a constant $C> 1$ such that for every $x\in X$ and $r>0$, $p(B(x,2r))\leq Cp(B(x,r))$ where $p(B(x,r)):=\sum_{y\in B(x,r)} p(y)$.
 \end{definition}

\paragraph{Why Doubling Condition?} Doubling dimension is introduced in \cite{Assouad1983} and \cite{Lar1967} which has become a popular notion of complexity measure of metric space. In Definition \ref{double}, a counterpart of this notion in a metric space $(X,d)$ endowed some distribution $p$ is given for our use. Generally, regarding $p$ as a measure of $X$, the "doubling condition" says every ball's volume is upper bounded by a universal constant times the volume of the ball with the same center but half radius. It measures the complexity of distribution $p$. The distribution satisfying the "doubling condition" somewhat has similar property as the uniform distribution on a compact set of Euclidean space.

Since uniform distribution on a compact set (e.g. hypercube $[0,1]^k$ or unit ball $B_d(0,1)$) of Euclidean space satisfies the doubling condition, as an interesting and important corollary (see Corollary \ref{corohyp}), we show the sample complexity of problem $\mathrm{WIT}([0,1]^d,L_2,\calU,\varepsilon)$ and $\mathrm{WIT}(B_d(0,1),L_2,\calU,\varepsilon)$ is $\tilde \Theta(\varepsilon^{-\max\{2,d/2\}})$.
\subsection{Other Related Work}

There are also recent papers regarding identity or uniformity testing beyond the classical problem of $L_1$-testing. \cite{BC17} presented the generalized uniformity testing problem which asks if a discrete distribution we are taking samples from is uniform on its support. \cite{2017arXiv170902087D} then investigated the exact sample complexity of this problem. On testing in other distribution distances, \cite{DK17} gave characterizations of the sample complexity of identity testing in a variety of distance besides $L_1$-distance.

The study of metric space has a long history, we refer to \cite{Deza09} as a complete and in-depth treatment of metric space. The doubling dimension is introduced in \cite{Assouad1983} and \cite{Lar1967}, and in theoretical computer science community, it's first used in the paper \cite{Clarkson97} regarding nearest neighbor search.

Chaining is an efficient way of proving union bound for a variety of possibly dependent variables. The study of chaining dates back to Kolmogorov's study of Brownian motion. \cite{talagrand14} is a highly suggested book regarding the application of chaining methods in modern probability theory. In recent years, the chaining method finds many applications in theoretical computer science, we refer to \cite{chaining2016} as an introduction of chaining methods in theoretical computer science.

\section{Preliminary}
Some notations go first. The $L_p$ norm of a vector in $R^n$ is defined to be $\|x\|_p:=(\sum_{i=1}^n x_i^p)^{1/p}$ and $[n]=\{1,2,...,n\}$.

We then define some notations about metric space. Let $(X,d)$ be a metric space where $X$ is a ground set and $d:X\times X\rightarrow \mathbb{R}_{\geq0}$ is a metric on $X$ which satisfies:
\begin{itemize}
\item$d(a,b)=d(b,a)$ for $a,b\in X$.
\item$d(a,a)=0$ for $a\in X$.
\item Triangle Inequality: $d(a,b)+d(b,c)\geq d(a,c)$ for $a,b,c\in X$.
\end{itemize}

The diameter of $(X,d)$ is defined as $\max_{a,b\in X} d(a,b)$. For a distribution $p$ supported on $X$, we mean by a sample from $p$ a point in $X$. For a subset $M\subset X$, $p(M):=\int_M dp=\sum_{x\in M}p(x)$.

The following classical definitions about $\varepsilon$-net and $\varepsilon$-packing are essential in this paper.

 \begin{definition}[$\varepsilon$-net, $\varepsilon$-packing and well separated $\varepsilon$-net] \label{net}
 Let $(X,d)$ be a metric space and $M\subset X$.
 A subset $N\subset M$ is called an $\varepsilon$-net of $M$ if
 $$
 \forall x\in M,\exists y\in N,d(x,y)\leq \varepsilon.
 $$

 A subset $P\subset M$ is called an $\varepsilon$-packing of $M$ if
 $$
 \forall x,y\in P,x\not=y, d(x,y)> \varepsilon.
 $$

 A subset $N\subset M$ is called a well separated $\varepsilon$-net of $M$ if it's an $\varepsilon$-net as well as an $\varepsilon$-packing of $M$.
 \end{definition}

 The following lemma shows the duality between $\varepsilon$-net and packing.
 \begin{lemma}\label{dualnet}
 (see e.g. \cite{van2014probability}) Let $(X,d)$ be a metric space. Let $N(X,d,\varepsilon)$ denote the minimum size of $\varepsilon$-net of $(X,d)$ and $P(X,d,\varepsilon)$ denote the maximum size of $\varepsilon$-packing $(X,d)$. Then we have
 $$
 N(X,d,\varepsilon)\leq P(X,d,\varepsilon)\leq N(X,d,\varepsilon/2).
 $$
 \end{lemma}

To acknowledge the great importance of the work by \cite{VV14}, we restate their core theorem here, and show how it implies other worst-case bounds.

\begin{theorem}[\cite{VV14}]\label{vv14thm}
There exists an algorithm such that, when given sample access to an unknown distribution $q$ and full description of $p$, both supported on $[n]$, it uses $O\left(\max\left\{\frac{1}{\veps},\frac{\left\|p_{-\veps/16}^{-\max}\right\|_{2/3}}{\veps^2}\right\}\right)$ samples from $q$ to distinguish $q=p$ from $\|p-q\|_1\geq\veps$ with success probability at least $2/3$. Moreover, any such algorithm requires $\Omega\left(\max\left\{\frac{1}{\veps},\frac{\left\|p_{-\veps}^{-\max}\right\|_{2/3}}{\veps^2}\right\}\right)$ samples from $q$.
\end{theorem}

The worst-case upper and lower bounds of Theorem~\ref{vv14thm} is given by $p=\calU$, the uniform distribution, where $\|\calU\|_{2/3}=\sqrt{n}=\max_{p}\|p\|_{2/3}$.

\begin{corollary}\label{vv14coro}
There exists an algorithm such that, when given sample access to an unknown distribution $q$ and full description of $p$, both supported on $[n]$, it uses $O(\sqrt{n}/\veps^2)$ samples from $q$ to distinguish $q=p$ from $\|p-q\|_1\geq\veps$ with success probability at least $2/3$. Moreover, any such algorithm requires $\Omega(\sqrt{n}/\veps^2)$ in the worst case, in the choices of $p$.
\end{corollary}

 \section{Wasserstein Identity Testing}\label{s3}

 First, we restate Theorem~\ref{informalone} and give its proof in two folds.

\begin{theorem} \label{gen}
Let $(X,d)$ be a metric space endowed with a distribution $p$. Let $D$ be its diameter. Let $\{N_i,\log \frac{\varepsilon}{8}\leq i\leq \log D\}$ be a sequence of well-separated $2^i$-net of $(X,d)$. There is an algorithm, given sample access to an unknown distribution $q$ over $X$ and $\varepsilon>0$,
 \begin{itemize}
 \item accepts with probability at least $2/3$ if $p=q$;
 \item rejects with probability at least $2/3$ if $\tran_{d}(p,q)\geq \varepsilon$.
 \end{itemize} The sample complexity of this algorithm is
 $$\tilde{O}\left(\max\left\{\frac{2^{2i}|N_i|^{1/2}}{\varepsilon^2}:\log \frac{\varepsilon}{8}\leq i\leq \log D\right\}\right). $$
 Moreover, any algorithm which distinguishes the two cases for any fixed $p$ and unknown $q$ with probability at least $2/3$ takes
 $${\Omega}\left(\max\left\{\frac{2^{2i}|N_i|^{1/2}}{\varepsilon^2}:\log \frac{\varepsilon}{8}\leq i\leq \log D\right\}\right)
 $$ many samples in the worst case of $p$.
\end{theorem}

\subsection{The Upper Bound}

The high level idea of our testing algorithm is by converting $(X,d)$ into a tree metric space $(T,d_T)$ with $X\subset T$ and $d\leq d_T$ when the latter is restricted on $X$, hence $\tran_d(p,q)\leq \tran_{d_T}(p,q)$. This means identity testing to $p$ in $\tran_{d_T}$ is at least as hard as in $\tran_d$, so a tester which works on $(X,d_T)$ also works on $(X,d)$. More specifically, we make use of $\varepsilon$-net of metric space $(X,d)$ to do the construction of $d_T$.

Recall that $\{N_i\}$ is a sequence of $2^i$-net of $(X,d)$. For each $y\in X$, we define $\pi_{i}(y)=\arg\min_{z\in N_i}d(z,y)$. Denote $l=\lfloor\log\frac{\varepsilon}{8}\rfloor$ and $r=\lceil\log D\rceil$.


We convert the metric space $(X,d)$ to a tree metric $(T,d_T)$ in the following way: Let $T=X\cup (\cup_{\log \frac{\varepsilon}{8}\leq i\leq \log D} N_i)$ (with replacement, see the figure below) where every $x\in X$ corresponds to a leaf of $T$. There are $\lceil\log D\rceil+\lfloor\log \frac{8}{\varepsilon}\rfloor$ many levels of internal nodes, every node in the $i$-th level of the tree represents a point in $N_i,\log \frac{\varepsilon}{8}\leq i\leq \log D$. For every leaf $x\in X$, add an edge $(x, \pi_{l}(x))$ with weight $2^l$. For each internal node $x\in N_i$, add an edge $(x, \pi_{i+1}(x))$ with weight $2^{i+1}$. Since the diameter of $(X,d)$ is $D$,  $N_r$ contains only one point which is the root of $T$. Define the tree metric $d_T(x,y)$ to be the sum of weights of edges in the unique shortest tree path from $x$ to $y$. Converting $p$ ($q$) into a distribution supported on $T$ such that it's supported on the leaves of $T$ with the same probability mass. With a little abuse of notation, we also use $p$ ($q$) to denote the transformed distribution on leaves of $T$.

\begin{definition}
For $x\in N_i$, let $\tilde p_i(x)$ (resp. $\tilde q_i(x)$) denote the sum of probability mass of all leaves in the subtree rooted at $x$. Then $\tilde p_i$, $\tilde q_i$ can be regarded as a distribution over $N_i$.
\end{definition}

Having defined the distributions induced by the well-separated $2^i$-nets, we are ready to give the algorithm that solves the problem below.

\begin{algorithm}[!h]
	\label{gen_algo}
	\caption{$\mathrm{WIT}$ over General Metric Space}
	\begin{algorithmic}[1] \label{algone}
		\REQUIRE Full description of metric space $(X,d)$, distribution $p$, proximity parameter $\varepsilon>0$ and sample access to $q$, well separated $2^i$-net sequence $\{N_i:\log \frac{\varepsilon}{8}\leq i\leq \log D\}$.
		\ENSURE  Accept with probability at least $2/3$ if $p=q$; reject with probability at least $2/3$ if $\tran_{d}(p,q)\geq \varepsilon$.
		\STATE Set $m=O\left(\log^3 \big{(} \frac{D}{\varepsilon} \big{)}\max\left\{\frac{2^{2i}|N_i|^{1/2}}{\varepsilon^2}:\log \frac{\varepsilon}{8}\leq i\leq \log D\right\}\right)$. Take $m$ i.i.d samples from $q$.
		\STATE For each $i$, using these samples to test whether $\tilde p_i=\tilde q_i$ or $\|\tilde p_i-\tilde q_i\|_1\geq {\Omega}(2^{-i}\varepsilon/\log(D/\veps))$. (By the $L_1$ tester given in Corollary~\ref{vv14coro})
		\RETURN Accept if all sub-tester accept, otherwise reject.
	\end{algorithmic}
\end{algorithm}

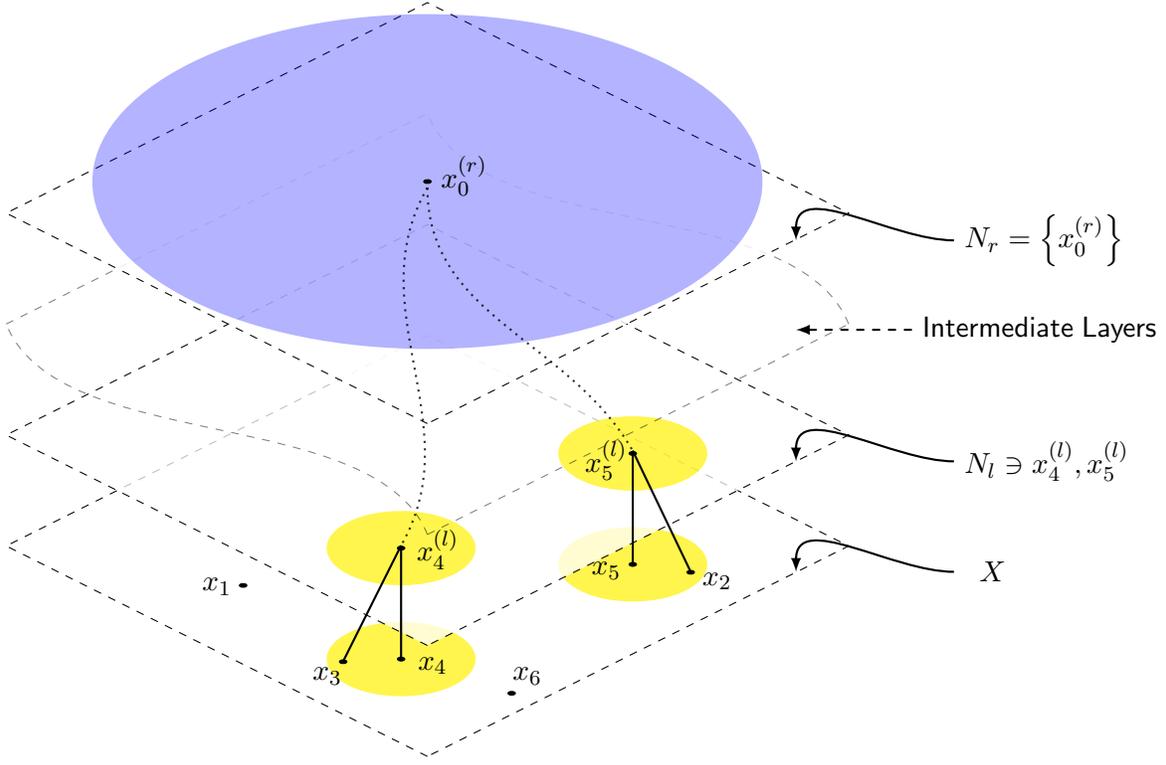
\begin{figure}[!h]
\centering
\begin{tikzpicture}[scale=.7,every node/.style={minimum size=1cm},on grid]
		
    \begin{scope}[
        yshift=-60,every node/.append style={
        yslant=0.5,xslant=-1},yslant=0.5,xslant=-1
                  ]
        \draw[black,dashed] (-1,-1) rectangle (7,7);
        \fill [fill=yellow,fill opacity=.7](.6,1.1) circle (1);
        \fill [fill=yellow,fill opacity=.7](4.6,0.7) circle (1);

        \draw [fill=black](1,-0.6) circle (.05) ;
        \draw [fill=black](0.5,4) circle (.05);
        \draw [fill=black](5,0) circle (.05);
        \draw [fill=black](0,1.6) circle (.05);
        
        \draw [fill=black](.6,1.1) circle (.05);
        \draw [fill=black](4.6,0.7) circle (.05);

    \end{scope} 
	
	\begin{scope}[
            yshift=0,every node/.append style={
            yslant=0.5,xslant=-1},yslant=0.5,xslant=-1
            ]
        \fill[white,fill opacity=.75] (-1,-1) rectangle (7,7);
        \draw[black,dashed] (-1,-1) rectangle (7,7);

        \fill [fill=yellow,fill opacity=.7](.6,1.1) circle (1);
        \fill [fill=yellow,fill opacity=.7](4.6,0.7) circle (1);
        
        \draw [fill=black](.6,1.1) circle (.05);
        \draw [fill=black](4.6,0.7) circle (.05);
        
    \end{scope}
    
            \begin{scope}[
        yshift=60,every node/.append style={
        yslant=0.5,xslant=-1},yslant=0.5,xslant=-1
                  ]
                  \draw[black,dashed,opacity=.5](-1,-1) -- (7,-1)
                  (-1,7) -- (7,7)
                  (-1,-1) .. controls (1,2) and (-3,4) .. (-1,7)
                  (7,-1) .. controls (9,2) and (5,4).. (7,7);
         \end{scope}
         
             \begin{scope}[
    	yshift=120,every node/.append style={
    	    yslant=0.5,xslant=-1},yslant=0.5,xslant=-1
    	             ]
                \fill[white,fill opacity=.85] (-1,-1) rectangle (7,7);
        \draw[black,dashed] (-1,-1) rectangle (7,7);

        \fill [fill=blue,fill opacity=.3](3.6,3.6) circle (4.5);
        \draw[fill=black](3.6,3.6) circle (.05);

    \end{scope}

    \draw[-latex,thick] (10,6.7) node[right]{$N_r=\left\{x_0^{(r)}\right\}$}
         to[out=180,in=90] (7,6.7);

    \draw[-latex,dashed,thick] (9.2,5) node[right]{$\mathsf{Intermediate\;Layers}$}
         to[out=180,in=0] (7,5);

    \draw[-latex,thick](10,2.5)node[right]{$N_l\ni x_4^{(l)},x_5^{(l)}$}
        to[out=180,in=90] (7,2.5);

    \draw[-latex,thick](10,0.4)node[right]{$X$}
        to[out=180,in=90] (7,0.4);	
    \fill[black]
            (0.2,0.1) node [above] {$x_4^{(l)}$}        
            (3.4,1.8) node [above] {$x_5^{(l)}$}
        
        (0.7,7.2) node [above] {$x_0^{(r)}$}

        (-4,-0.6) node [above] {$x_1$}
        (0.1,-2.1) node [above] {$x_4$}
        (-1.9,-2.3) node [above] {$x_3$}
        (1.9,-2.3) node [above] {$x_6$}
        (3.4,-0.3) node [above] {$x_5$}
        (5.5,-0.5) node [above] {$x_2$};	
        
      \draw [thick](-0.5,0.9) node {} to (-0.5,-1.3)
      	 (-0.5,0.9) node {} to (-1.6,-1.3)
	 (3.9,2.7) node {} to (3.9,0.5)
	 (3.9,2.7) node {} to (5.0,0.4);
      \draw[dotted,thick,opacity=.8](0,7.7) node {} to [out=240,in=60](-0.5,0.9)
	 (0,7.7) node {} to [out=270,in=120](3.9,2.7);
	 
\end{tikzpicture}
\caption{\small The abstract structure of $T$. Every plane represents a level $N_i$ in $T$. To avoid abuse of notation, we add a superscript for nodes higher than the bottom level, indicating the level they lie in. Filled ovals stand for the balls $B(z^{(i)},2^i)$. }
\end{figure}

 \begin{lemma}\label{lemd_t}If $p=q$ then $\tran_{d_T}(p,q)=0$. Moreover,
\begin{eqnarray}
\tran_{d_T}(p,q)\geq \tran_{d}(p,q) \label{one}
\end{eqnarray}
 \end{lemma}

 \begin{proof}[Proof of Lemma~\ref{lemd_t}] The construction is deterministic, so we know $p=q$ implies $\tran_{d_T}(p,q)=0$. To prove $\tran_{d_T}(p,q)\geq \tran_{d}(p,q)$, we only need to show $d_T(x,y)\geq d(x,y)$ for every $x,y\in X$. Assume the \emph{lowest common ancestor} of $x,y$ in $T$ is in the $j$-th level and all internal nodes along the unique tree path from $x$ to $y$ is $z_l,z_{l+1},...,z_j=w_j,w_{j-1},...,w_l$ where $z_i,w_i\in N_i$. So by triangle inequality and the construction of $T$,
 \begin{eqnarray*}
d(x,y)&\leq& d(x,z_l)+\sum_{i=0}^{j-l-1}d(z_{l+i},z_{l+i+1})+\sum_{i=0}^{j-l-1}d(w_{j-i},w_{j-i-1})+d(w_l,y)\\
 &\leq&2^l+2^{l+1}+...2^j+2^j+2^{j-1}+...+2^l\\
 &=&d_T(x,z_l)+\sum_{i=0}^{j-l-1}d_T(z_{l+i},z_{l+i+1})+\sum_{i=0}^{j-l-1}d_T(w_{j-i},w_{j-i-1})+d_T(w_l,y)\\
 &=&d_T(x,y)\\
 \end{eqnarray*}
 \end{proof}

We have the following simple characterization of Wasserstein distance w.r.t. $d_T$. This lemma shows that actually, we can convert the the problem $\mathrm{WIT}(T,d_T,p,\varepsilon)$ to some sub-problems in $L_1$-distance.

\begin{lemma}\label{lemdecom}
  \begin{eqnarray}
 \tran_{d_T}(p,q)=2^l\mathcal{L}_1(p,q)+\sum_{i=l}^{r-1} 2^{i+1} \mathcal{L}_1(\tilde p_{i},\tilde q_{i}) \label{two}
 \end{eqnarray}
 where $\mathcal{L}_1(\cdot,\cdot)$ is the $L_1$ distance between two probability distributions with the same support.
 \end{lemma}

 \begin{proof}[Proof of Lemma~\ref{lemdecom}]
Consider an edge $e$ which connects a node $x$ in the $i$th-level and its father, it has weight of $2^{i+1}$. Since the probability mass of $p$ and $q$ on the leaves inside the subtree rooted $x$ differ by $|\tilde p_{i}(x)-\tilde q_{i}(x)|$, hence there is exactly $|\tilde p_{i}(x)-\tilde q_{i}(x)|$ probability mass transported along $e$ which produces the cost $2^{i+1}|\tilde p_{i}(x)-\tilde q_{i}(x)|$ in Wasserstein distance. Summing over all edges, we have
  $$
 \tran_{d_T}(p,q)=2^l\mathcal{L}_1(p,q)+\sum_{i=l}^{r-1} 2^{i+1} \mathcal{L}_1(\tilde p_{i},\tilde q_{i}).
 $$
 where we note that every leaf $z\in X$ has an edge incident on it with weight $2^l$.
 \end{proof}

Now we can prove the correctness of Algorithm \ref{algone}.\\

\begin{proof}[Proof of Upper Bound in Theorem~\ref{gen}]
By Corollary~\ref{vv14coro} and the median trick, $O\left(\frac{|N_i|^{1/2}}{\varepsilon_i^2}\log\frac{1}{\delta}\right)$ many samples suffice to test $\tilde p_i=\tilde q_i$ versus $\mathcal{L}_1(\tilde p_{i},\tilde q_{i})\geq \Omega(\varepsilon_i)$ with probability at least $1-\delta$. Choose $\delta$ such that $(r-l+1)\delta\leq 1/3$, then by union bound, with probability at least $2/3$, all testers succeed, and when all sub-testers succeed, we are guaranteed to report a correct answer.

When $p=q$, we have $\tilde p_i(x)=\tilde q_i(x)$ for each $i$ and $x$, thus with probability at least $2/3$, every sub-tester accepts.

When $\tran_{d}(p,q)\geq \varepsilon$, by $(\ref{one})$ and $(\ref{two})$, one has
$$\sum_{i=l}^{r-1}2^{i+1}\mathcal{L}_1(\tilde p_i,\tilde q_i)\geq\tran_{d_T}(p,q)-2^l\mathcal{L}_1(p,q)\geq\veps-2\veps/8=\Omega(\veps)$$
therefore, there is some $i$ such that $\mathcal{L}_1(\tilde p_{i},\tilde q_{i})\geq \Omega(2^{-i}\varepsilon/(r-l))$, so the corresponding tester rejects, and the algorithm rejects, with probability at least $1-\delta\geq2/3$. To satisfy the sample complexity of all sub-testers, the overall upper bound is finally given by,
 $$O\left(\max\left\{(r-l)^3\frac{2^{2i}|N_i|^{1/2}}{\varepsilon^2}:l\leq i\leq r-1\right\}\right)=\tilde{O}\left(\max\left\{\frac{2^{2i}|N_i|^{1/2}}{\varepsilon^2}:\log \frac{\varepsilon}{8}\leq i\leq \log D\right\}\right). $$
\end{proof}

\subsection{Lower Bound over General Metric Space}

We prove the worst-case lower bound of sample complexity for the problem $\mathrm{WIT}(X,d,p,\veps)$, which completes the proof of Theorem~\ref{gen}.

\begin{proof}[Proof of Lower Bound in Theorem \ref{gen}] Let $i=\arg\max \{\frac{2^{2j}|N_j|^{1/2}}{\varepsilon^2}:\log \frac{\varepsilon}{8}\leq j\leq \log D\})$. Denote $n$ the number of points in $N_i$ and $N_i=\{x_1,x_2,...x_n\}\subset X$. We then show how to convert the identity testing problem on $[n]$ in $L_1$ distance to the Wasserstein identity testing problem on $(X,d)$.

On testing if an unknown distribution $u$ is identical to $v$ in $L_1$-distance where $u$ and $v$ are supported on $[n]$. We make the following transformation: let $v'$ be a distribution supported on $N_i\subset X$ such that $v'(x_j)=v(j),j\in [n]$, and construct a distribution $u'$ supported on $N_i$ by using $u$ such that a sample $j$ from $u$ is mapped into a sample $x_j$. Hence $u'(x_j)=u(j)$ by construction.

So if $u=v$ then $u'=v'$ while if $\mathcal{L}_1(u,v)\geq \varepsilon$ then $\tran_d (u', v')\geq 2^{i-1}\varepsilon$ (recall $N_i$ is an $2^i$-packing of $X$). So if we can test $u'=v'$ versus $\tran_d (u', v')\geq2^{i-1}\varepsilon$ by using $o(\frac{4^i\sqrt{n}}{(2^{i-1}\varepsilon)^2})=o(\frac{\sqrt{n}}{\veps^2})$ samples, we can distinguish $u=v$ from $\mathcal{L}_1(u,v)\geq \varepsilon$ by using $o(\frac{\sqrt{n}}{\varepsilon^2})$ samples, which contradicts the existing worst-case  $\Omega(\frac{\sqrt{n}}{\varepsilon^2})$ lower bound in Corollary~\ref{vv14coro}.

Hence an algorithm which solves the problem $\mathrm{WIT}(X,d,p,\varepsilon)$ for every distribution $p$ uses at least
$$
\Omega\bigg{(}\max\left\{\frac{2^{2i}|N_i|^{1/2}}{\varepsilon^2}:\log \frac{\varepsilon}{8}\leq i\leq \log D\right\}\bigg{)}
$$
many samples in the worst case (over the choices of distribution $p$).
\end{proof}

\subsection{Nearly Optimal Instance Sample Complexity Provided the "Doubling Condition"}
In this section, we characterize nearly optimal instance sample complexity of Problem \ref{pro}, by additionally assuming that $(X,d,p)$ satisfies the "Doubling Condition" (Definition \ref{double}). For convenience, we define some new notations.

 \begin{definition} \label{tech}
 Let $(X,d)$ be a metric space endowed with a distribution $p$. Assume $D$ is the diameter of $(X,d)$, and for evert $\log \frac{\varepsilon}{8}\leq i\leq \log D$, $N_i$ is a well separated $2^i$-net of $(X,d)$.

 For every $y\in X$, define $\pi_i(y)=\mathrm{argmin}_{z\in N_i} d(z,y)$. For every $x\in N_i$, define the clustering of $x$ to be $C_i(x)=\{y\in X: \pi_i(y)=x\}$. Let $p_i(x):=\sum_{y\in C_i(x)} p(y)$. Let $N_i=\{x_1,x_2,...,x_n\}$. we can regard $p_i(\cdot)$ as a discrete distribution on $[n]$. With a little abuse of notation, for every $j\in[n]$, let $p_i(j)=p_i(x_j)=\sum_{y\in C_i(x_j)} p(y)$.

 \end{definition}

 \begin{lemma} \label{techlemma}
 Assume $(X,d)$ is a metric space endowed with a probability distribution $p$, $N_i$ is a well separated $2^i$-net and $C_i,p_i$ are constructed as in Definition \ref{tech}. Assume $N_i=\{x_1,x_2,...,x_n\}$, then for every $j\in [n]$,
 \begin{eqnarray}
 p(B(x_j,2^{i-1}))\leq p_i(j)\leq p(B(x_j,2^i)). \label{ulb}
 \end{eqnarray}
 \end{lemma}

 \begin{proof}[Proof of Lemma~\ref{techlemma}]
 We only need to prove that $B(x_j,2^{i-1})\subset C_i(x_j)\subset B(x_j,2^i)$. Recall that $N_i$ is a $2^i$-net as well as $2^i$-packing of $(X,d)$.

 For every $x\in B(x_j,2^{i-1})$, if $x$ is clustered to some $x_k,k\not=j$, then by definition $d(x,x_k)\leq d(x,x_j)\leq 2^{i-1}$. Hence we have $d(x_k,x_j)\leq d(x,x_k)+d(x,x_j)\leq 2^{i-1}+2^{i-1}=2^i$ which contradicts the fact that $N_i$ is a $2^i$-packing. So we have, $B(x_j,2^{i-1})\subset C_i(x_j)$.

 For every $y\not\in B(x_j,2^i)$, note that $N_i$ is a $2^i$-net, so there is some $k\in [n]$ such that $d(x_k,y)\leq 2^i<d(x_j,y)$ which means $y\not\in C_i(x_j)$. So we have $C_i(x_j)\subset B(x_j,2^i)$.
 \end{proof}

\begin{remark}
The reader may have a natural question, why do we use $p_i$ instead of $\tilde{p}_i$ defined in the proof of Theorem~\ref{gen}? From a technical perspective, our answer is that we cannot obtain upper and lower bound for $\tilde{p}_i$ as good as (\ref{ulb}), which will be essential in the proof of the instance lower bound.
\end{remark}

 \begin{theorem} \label{instance}
 Let $(X,d)$ be a metric space endowed with a distribution $p$ and $\varepsilon>0$ provided $(X,d,p)$ satisfies the "doubling condition" Definition \ref{double}. Let $\{N_i,\log \frac{\varepsilon}{8}\leq i\leq \log D\}$ be a sequence of well separated $2^i$-net of $(X,d)$. There is an algorithm, given sample access to an unknown distribution $q$ over $X$,
 \begin{itemize}
 \item accepts with probability at least $2/3$ if $p=q$;
 \item rejects with probability at least $2/3$ if $\tran_d(p,q)\geq \varepsilon$.
 \end{itemize} Let $p_i()$ be as defined in Definition \ref{tech} then the sample complexity of this algorithm is
 $$\tilde{O}\bigg{(}\max \bigg{\{}\max\bigg{\{}2^{2i}\varepsilon^{-2} |{p}_i(\cdot)|_{2/3},2^i\varepsilon^{-1}\bigg{\}}:\log \frac{\varepsilon}{8}\leq i\leq\log D\bigg{\}}\bigg{)}.
 $$

 Moreover, the following is a sample complexity lower bound for this task.
 $${\Omega}\bigg{(}\max \bigg{\{}\max\bigg{\{}2^{2i} \varepsilon^{-2}|{p}_i(\cdot)^{-max}_{-2^{-i}\varepsilon}|_{2/3},2^i\varepsilon^{-1}\bigg{\}}:\log \frac{\varepsilon}{8}\leq i\leq\log D\bigg{\}} \bigg{)}
 $$ many samples. Here ${p}_i(\cdot)^{-max}_{-2^{-i}\varepsilon}$ represents the probability vector obtained by removing element with the largest probability mass and keeping moving the element with the smallest probability mass until $2^{-i}\varepsilon$ mass is removed.
 \end{theorem}

 \begin{algorithm}[!h]
	\caption{Instance $\mathrm{WIT}$}
	\begin{algorithmic}[1] \label{algtwo}
		\REQUIRE Full description of metric space $(X,d)$, distribution $p$, proximity parameter $\varepsilon>0$ and sample access to $q$, well separated $\varepsilon$-net sequence $\{N_i:\log \frac{\varepsilon}{8}\leq i\leq \log D\}$.
		\ENSURE  Accept with probability at least $2/3$ if $p=q$; reject with probability at least $2/3$ if $\tran_{d}(p,q)\geq \varepsilon$.
		\STATE Set $m=\tilde{O}\bigg{(}\max \bigg{\{}\max\bigg{\{}2^{2i}\varepsilon^{-2} |{p}_i(\cdot)^{-max}_{-2^{-i-4}\varepsilon}|_{2/3},2^i\varepsilon^{-1}\bigg{\}}:\log \frac{\varepsilon}{8}\leq i\leq\log D\bigg{\}}\bigg{)}$.
        \STATE Take $m$ i.i.d samples from $q$.
		\STATE For each $i$, using these samples to test whether $p_i=q_i$ or $\|p_i-q_i\|_1\geq \Omega(2^{-i}\varepsilon/\log(D/\veps))$. (By the instance-optimal $L_1$ tester given in Theorem~\ref{vv14thm})
		\RETURN Accept if all sub-tester accept, otherwise reject.
	\end{algorithmic}
\end{algorithm}

\begin{proof}[Proof of Theorem~\ref{instance}]
 Firstly, we prove the upper bound, which is relatively simpler. We proceed as in the proof of Theorem \ref{gen} to construct distributions $\tilde{p}_i$ for every $\log \frac{\varepsilon}{8}\leq i\leq \log D$. Then we use an instance version of Algorithm \ref{algone} by using instance optimal version $L_1$ subtester from Theorem~\ref{vv14thm} instead of the worst case version. By Theorem~\ref{vv14thm} and the union bound, we know that $m$ samples can guarantee each subtester succeed with probability at least $2/3$, then Algorithm \ref{algtwo} works by the same reason.

  The only remaining work is to convert $\tilde{p}_i$ into $p_i(\cdot)$ in the sample complexity. Recall that for $x\in N_i$, $\tilde{p}_i(x)$ is the sum of probability mass on the leaves inside the subtree rooted at $x$. Now note that for any such leaf $y$, by construction, $d(x,y)\leq 2^{l}+...+2^i\leq 2^{i+1}$ which means every leaf inside the subtree root at $x$ is contained in the ball $B(x,2^{i+1})$. So by doubling condition and Lemma \ref{techlemma},
  \begin{eqnarray}
  \tilde{p}_i(x)\leq p(B(x,2^{i+1}))\leq Cp(B(x,2^i))\leq C^2p(B(x,2^{i-1}))\leq C^2 p_i(x).
  \end{eqnarray}
   Since $C$ is a universal constant,  we can bound $|\tilde{p}_i(\cdot)^{-max}_{-2^{-i-4}\varepsilon}|_{2/3}$ by $|p_i(\cdot)|_{2/3}$ in big-O notation.

 Now we turn to the proof of lower bound. By Theorem~\ref{vv14thm}, we know any algorithm which tests the identity to $p_{i}$ in $\mathcal{L}_1$ distance with proximity parameter $2^{-{i}}\varepsilon$ requires
 $$
 \Omega(\max\{2^{-i}\varepsilon^{-1},\varepsilon^{-2}2^{2i}|{p_i}^{-max}_{-\varepsilon}|_{2/3}:i\in [\log \frac{\varepsilon}{8},\log D]\})
 $$
 many samples.

 For an unknown discrete distribution $q$ on $\big{[}|N_i(x)|\big{]}$, we show how to convert $q$ to a distribution $q^*$ on $X$ so as we can reduce the problem of identity testing to $p_i$ in $L_1$ distance to the problem of identity testing to $p$ in Wasserstein distance. We assume $N_i=\{x_1,x_2,...,x_n\}$ in what follows.

 Being precise, fix a $y\in X\backslash\bigcup_{j\in [n]} B(x_j,2^{i-1})$, every time we need to take a sample from $q^*$, we do the following: first take a sample $j\in[n]$ from $q$. With probability $\frac{p(B(x_j,2^{i-1}))}{p_i(j)}$, $x_j$ is regarded as the sample of $q^*$. With probability
 $1-\frac{p(B(x_j,2^{i-2}))}{p_i(j)}$, $y$ is regarded as the sample of $q^*$.

Obviously we have that $q^*(B(x_j,2^{i-1}))=q^*(x_j)=\frac{p(B(x_j,2^{i-1}))}{p_i(j)} q(j)$.

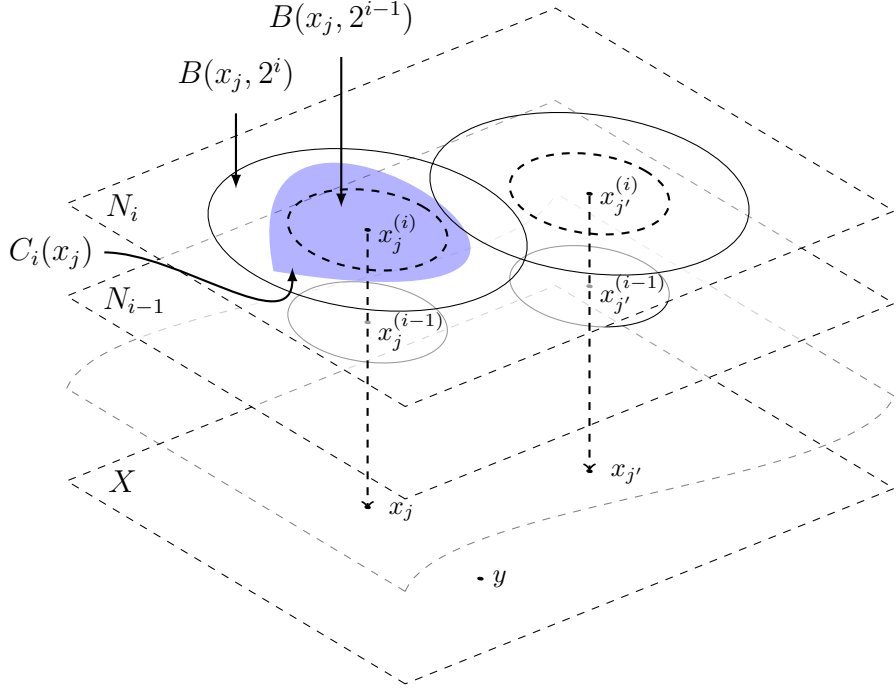
\begin{figure}[!h]
\centering
\begin{tikzpicture}[scale=.5,every node/.style={minimum size=1cm},on grid]
		
   \begin{scope}[
        yshift=-210,every node/.append style={
        yslant=0.4,xslant=-1},yslant=0.4,xslant=-1
                  ]
                          \draw[black,dashed] (-4,-2) rectangle (9,7);

	\draw[fill=black] (.1,3.1) circle (.05)
			(0,0) circle (.05)
			(4.6,1.7) circle (.05);
     \end{scope}
    
    \begin{scope}[
        yshift=-140,every node/.append style={
        yslant=0.4,xslant=-1},yslant=0.4,xslant=-1
                  ]
        \draw[black,dashed,opacity=.5](-4,-2) .. controls (-1,0) and (6,-4) .. (9,-2)
                  (-4,7) .. controls (-1,9) and (6,5) .. (9,7)
                  (-4,-2) -- (-4,7)
                  (9,-2) -- (9,7);
     \end{scope}

    \begin{scope}[
        yshift=-70,every node/.append style={
        yslant=0.4,xslant=-1},yslant=0.4,xslant=-1
                  ]
                \fill[white,fill opacity=.6] (-4,-2) rectangle (9,7);

        \draw[black,dashed] (-4,-2) rectangle (9,7);
        
        \draw [thin,black](.1,3.1) circle (1.5)
        				(4.6,1.7) circle (1.5);

	\draw[fill=black] (.1,3.1) circle (.05)
			(4.6,1.7) circle (.05);

    \end{scope} 
	
	\begin{scope}[
            yshift=0,every node/.append style={
            yslant=0.4,xslant=-1},yslant=0.4,xslant=-1
            ]
        \fill[white,fill opacity=.6] (-4,-2) rectangle (9,7);
        \draw[black,dashed] (-4,-2) rectangle (9,7);
        
	\fill[fill=blue,fill opacity=.3] (-2.5,3) .. controls (4,11) and (3,-4) .. (-2.5,3);
        \draw [black,thin](.1,3.1) circle (3)
        				(4.6,1.7) circle (3);
	\draw [dashed,thick](.1,3.1) circle (1.5)
        				(4.6,1.7) circle (1.5);
        
	\draw[fill=black] (.1,3.1) circle (.05)
			(4.6,1.7) circle (.05);
        
    \end{scope}



        \draw[-latex,thick,font=\large](-6.5,5)node[above]{$B(x_j,2^i)$} to[out=270,in=90] (-6.5,3);
        \draw[-latex,thick,font=\large](-3.7,6.5)node[above]{$B(x_j,2^{i-1})$} to[out=270,in=90] (-3.7,2.5);
        \draw[-latex,thick,font=\large](-10,1.3)node[left]{$C_i(x_j)$} to[out=0,in=270] (-5,.9);
    
    
    	\fill[black,font=\large] (-8.5,2.5) node [left] {$N_i$}
					(-8.1,0) node [left] {$N_{i-1}$}
					(-8.5,-4.8) node [left] {$X$};
					
	\fill[black,font=\small] (-3.2,1.8) node[right] {$x_j^{(i)}$}
					(2.7,2.8) node[right] {$x_{j'}^{(i)}$}
					(-3,-0.8) node[right] {$x_j^{(i-1)}$}
					(2.9,0.2) node[right] {$x_{j'}^{(i-1)}$}
					(-3.1,-5.6) node[right] {$x_j$}
					(2.9,-4.6) node[right] {$x_{j'}$}
					(-0.5,-7.4)node[right] {$y$};
					
	\draw[->,dashed,thick](-3,1.8)node{}to(-3,-5.5);
	\draw[->,dashed,thick](2.9,2.8)node{}to(2.9,-4.5);
	
\end{tikzpicture}

\caption{\small An illustration of $B(x_j,2^i),B(x_j,2^{i-1})$ and $C_i(x_j)$. As shown in the proof, the balls $B(x_j,2^{i-1})$ for $x_j\in N_i$ are disjoint. In the construction of $q^*$, the probability mass is only placed upon $x_j$s in $N_i$, and some $y\in X$ which is not inside any $B(x_j,2^{i-1})$.}
\end{figure}

 \begin{lemma} \label{techlemma2}
 \begin{eqnarray}
 \tran_d (p,q^*) \geq \sum_{j\in [n]} 2^{i-1} |p(B(x_j,2^{i-1}))-q^*(B(x_j,2^{i-1}))|
 \end{eqnarray}
 \end{lemma}

 \begin{proof}[Proof of Lemma~\ref{techlemma2}]
 The Wasserstein distance $\tran_d (p,q^*)$ is the cost of transporting probability mass from $p$ to $q^*$. Recall that $N_i=\{x_1,x_2,...,x_n\}$ is a $2^i$-packing, hence the ball $B(x_j,2^{i-1}),j\in [n]$ doesn't intersect with each other. That means for every $j\in [n]$, $|p(B(x_j,2^{i-1}))-q^*(B(x_j,2^{i-1}))|$ much probability mass needs transporting into or out of the ball $B(x_j,2^{i-1})$. Noting that by construction, $q^*(x_j)=q^*(B(x_j,2^{i-1}))$, the probability mass of $q^*(B(x_j,2^{i-1}))$ is concentrated on $x_j$, hence the cost of transporting per unit probability mass is at least $2^{i-1}$. Summing over $j\in [n]$, we have,
 \begin{eqnarray}
 \tran_d (p,q^*) \geq \sum_{j\in [n]} 2^{i-1} |p(B(x_j,2^{i-1}))-q^*(B(x_j,2^{i-1}))|.
 \end{eqnarray}
 \end{proof}

  So by Lemma \ref{techlemma2}, the construction, the doubling condition and Lemma \ref{techlemma} respectively,
 \begin{eqnarray*}
 \tran_d (p,q^*) &\geq &\sum_{j\in [n]} 2^{i-1} |p(B(x_j,2^{i-1}))-q^*(B(x_j,2^{i-1}))|\\
 &=&\sum_{j\in [n]} 2^{i-1} p(B(x_j,2^{i-1})))|1-q(j)/p_i(j)|\\
 &\geq&\sum_{j\in [n]} 2^{i-1} C^{-1}p(B(x,2^i))|1-q(j)/p_i(j)|\\
 &\geq&\sum_{j\in [n]} 2^{i-1} C^{-1} |p_i(j)-q(j)|\\
 &=&2^{i-1}C^{-1}\mathcal{L}_1(p_i,q)
 \end{eqnarray*}

  So the problem $\mathrm{WIT}(X,d,p,\varepsilon)$ is at least as hard as the identity testing to $p_i$ in ${L}_1$ distance with proximity parameter $\Omega(2^{-i}\varepsilon)$ since we can reduce the latter to the former. So we conclude that any algorithm for the former task takes at least
  $$
 \Omega(\max\{{2^{i}\varepsilon}^{-1},{\varepsilon^{-2}}{2^{2i}|{p}_{i}(\cdot)^{-max}_{-\varepsilon}|_{2/3}}:\log \frac{\varepsilon}{8}\leq i\leq\log D]\})
  $$
  many samples.
\end{proof}

\begin{remark}
As we have seen in the proof, technically, the "doubling condition" is essential in proving $\tran_d (p,q^*)\geq \Omega(2^{-i} \mathcal{L}_1(p_i,q))$.
\end{remark}

We immediately obtain the following result as a corollary.

\begin{corollary} [Uniformity Testing on Hypercube and unit ball]\label{corohyp}
Assume $d$ is a constant.
Let $q$ be an unknown distribution over the $d$-dimensional hypercube $[0,1]^d$. Then there is an algorithm, given $\tilde{O}({\varepsilon^{-\max\{\frac{d}{2} ,2\}}})$ samples from $q$,
 \begin{itemize}
 \item If $q$ is uniform on $[0,1]^d$ then accepts with probability $2/3$.
 \item If $q$ is $\varepsilon$-far from uniformity in Euclidean Wasserstein distance then rejects with probability $2/3$.
  \end{itemize}

  On the other hand, any algorithm for this task requires $\Omega\left(\varepsilon^{-\max\{2,\frac{d}{2}\}}\right)$ many samples. Moreover, replacing $[0,1]^d$ with the unit ball $B_d(0,1)$, the sample complexity is still $\tilde{O}({\varepsilon^{-\max\{\frac{d}{2} ,2\}}})$.
  \end{corollary}

  \begin{proof}[Proof of Corollary~\ref{corohyp}]
  Since $d$ is a constant, the Euclidean distance and the $L_{\infty}$ distance are equivalent. So we assume the metric is $L_{\infty}$ distance. Choosing the trivial well separated $2^i$-net for $\log \frac{\varepsilon}{8}\leq i\leq 0$,
  $$
  N_i=\{(k_12^i,k_22^i,...,k_d2^i):k_j=0,1,2,...,2^{-i},j\in [d]\}.
  $$
  Let $p$ be the uniform distribution on $[0,1]^d$ and $p_i$ be the distribution defined in Definition \ref{tech}. By computation, we have,
  $|p_i()|_{2/3}=\Theta(2^{-id/2})$ and ${p}_i(\cdot)^{-max}_{-2^{-i}\varepsilon}=\Theta({2^{-id/2}})$. So by Theorem \ref{instance} and simple computation, the sample complexity is
  \begin{eqnarray}
  \tilde{\Theta}(\varepsilon^{-\max\{2,d/2\}}). \label{ins}
  \end{eqnarray}

  Since $d$ is a constant, on the same way we can prove the sample complexity of Wasserstein identity testing in $B_d(0,1)$ is also given by (\ref{ins}).
  \end{proof}

\section{Future Work}
The most obvious direction of future work is to remove the "Doubling Condition" and
provide the instance optimal sample complexity bound for Wasserstein Identity Testing in arbitrary metric space.

A related problem is to consider another large family of distribution distance characterized by maximum mean discrepancy over functions in the unit ball of Reproducing Kernel Hilbert Space (RKHS) (see e.g., \cite{Gretton2012}).

\acks{The authors thank Dan Feldman and Jian Li for several insightful discussions.}

\bibliography{probdb}

\begin{thebibliography}{23}
\providecommand{\natexlab}[1]{#1}
\providecommand{\url}[1]{\texttt{#1}}
\expandafter\ifx\csname urlstyle\endcsname\relax
  \providecommand{\doi}[1]{doi: #1}\else
  \providecommand{\doi}{doi: \begingroup \urlstyle{rm}\Url}\fi

\bibitem[Arjovsky et~al.(2017)Arjovsky, Chintala, and
  Bottou]{arjovsky2017wasserstein}
Martin Arjovsky, Soumith Chintala, and L{\'e}on Bottou.
\newblock Wasserstein gan.
\newblock \emph{arXiv preprint arXiv:1701.07875}, 2017.

\bibitem[Assouad(1983)]{Assouad1983}
Patrice Assouad.
\newblock Plongements lipschitziens dans $r^n$.
\newblock \emph{Bulletin de la SocišŠtšŠ MathšŠmatique de France},
  111:\penalty0 429--448, 1983.
\newblock URL \url{http://eudml.org/doc/87452}.

\bibitem[at~Harvard(2016)]{chaining2016}
Seminar at~Harvard.
\newblock \emph{Chaining Methods and their Applications to Computer Science}.
\newblock 2016.
\newblock URL \url{https://toc.seas.harvard.edu/cmacs}.

\bibitem[Batu and Canonne(2017)]{BC17}
Tugkan Batu and Cl{\'{e}}ment~L. Canonne.
\newblock Generalized uniformity testing.
\newblock \emph{CoRR}, abs/1708.04696, 2017.
\newblock URL \url{http://arxiv.org/abs/1708.04696}.

\bibitem[Chan et~al.(2014)Chan, Diakonikolas, Valiant, and Valiant]{CDVV14}
Siu-On Chan, Ilias Diakonikolas, Gregory Valiant, and Paul Valiant.
\newblock Optimal algorithms for testing closeness of discrete distributions.
\newblock In \emph{Proceedings of the twenty-fifth annual ACM-SIAM symposium on
  Discrete algorithms}, pages 1193--1203. Society for Industrial and Applied
  Mathematics, 2014.

\bibitem[Clarkson(1997)]{Clarkson97}
Kenneth~L. Clarkson.
\newblock Nearest neighbor queries in metric spaces.
\newblock In \emph{Proceedings of the Twenty-ninth Annual ACM Symposium on
  Theory of Computing}, STOC '97, pages 609--617, New York, NY, USA, 1997. ACM.
\newblock ISBN 0-89791-888-6.
\newblock \doi{10.1145/258533.258655}.
\newblock URL \url{http://doi.acm.org/10.1145/258533.258655}.

\bibitem[Daskalakis et~al.(2017)Daskalakis, Kamath, and Wright]{DK17}
Constantinos Daskalakis, Gautam Kamath, and John Wright.
\newblock Which distribution distances are sublinearly testable?
\newblock \emph{CoRR}, abs/1708.00002, 2017.
\newblock URL \url{http://arxiv.org/abs/1708.00002}.

\bibitem[Devroye and Lugosi(2001)]{devroye2001combinatorial}
L.~Devroye and G.~Lugosi.
\newblock \emph{Combinatorial Methods in Density Estimation}.
\newblock Springer Series in Statistics. Springer New York, 2001.
\newblock ISBN 9780387951171.
\newblock URL \url{https://books.google.com/books?id=jvT-sUt1HZYC}.

\bibitem[Deza and Laurent(2009)]{Deza09}
Michel~Marie Deza and Monique Laurent.
\newblock \emph{Geometry of Cuts and Metrics}.
\newblock Springer Publishing Company, Incorporated, 1st edition, 2009.
\newblock ISBN 3642042945, 9783642042942.

\bibitem[D.G.Larman(1967)]{Lar1967}
D.G.Larman.
\newblock A new theory of dimension.
\newblock \emph{Proc.London Math.Soc.}, page~17, 1967.

\bibitem[Diakonikolas and Kane(2016)]{iliasnew16}
Ilias Diakonikolas and Daniel~M. Kane.
\newblock A new approach for testing properties of discrete distributions.
\newblock \emph{CoRR}, abs/1601.05557, 2016.
\newblock URL \url{http://arxiv.org/abs/1601.05557}.

\bibitem[Diakonikolas et~al.(2014)Diakonikolas, Kane, and
  Nikishkin]{DBLP:journals/corr/DiakonikolasKN14}
Ilias Diakonikolas, Daniel~M. Kane, and Vladimir Nikishkin.
\newblock Testing identity of structured distributions.
\newblock \emph{CoRR}, abs/1410.2266, 2014.
\newblock URL \url{http://arxiv.org/abs/1410.2266}.

\bibitem[Diakonikolas et~al.(2017)Diakonikolas, Kane, and
  Stewart]{2017arXiv170902087D}
Ilias Diakonikolas, Daniel~M Kane, and Alistair Stewart.
\newblock Sharp bounds for generalized uniformity testing.
\newblock \emph{arXiv preprint arXiv:1709.02087}, 2017.

\bibitem[Goldreich and Ron(2011)]{GR00}
Oded Goldreich and Dana Ron.
\newblock On testing expansion in bounded-degree graphs.
\newblock In \emph{Studies in Complexity and Cryptography. Miscellanea on the
  Interplay between Randomness and Computation}, pages 68--75. Springer, 2011.

\bibitem[Goldreich et~al.(1998)Goldreich, Goldwasser, and Ron]{OSD98}
Oded Goldreich, Shari Goldwasser, and Dana Ron.
\newblock Property testing and its connection to learning and approximation.
\newblock \emph{Journal of the ACM (JACM)}, 45\penalty0 (4):\penalty0 653--750,
  1998.

\bibitem[Gretton et~al.(2012)Gretton, Borgwardt, Rasch, Sch{\"o}lkopf, and
  Smola]{Gretton2012}
Arthur Gretton, Karsten~M Borgwardt, Malte~J Rasch, Bernhard Sch{\"o}lkopf, and
  Alexander Smola.
\newblock A kernel two-sample test.
\newblock \emph{Journal of Machine Learning Research}, 13\penalty0
  (Mar):\penalty0 723--773, 2012.

\bibitem[Li et~al.(2015)Li, Rabani, Schulman, and Swamy]{li2015learning}
Jian Li, Yuval Rabani, Leonard~J Schulman, and Chaitanya Swamy.
\newblock Learning arbitrary statistical mixtures of discrete distributions.
\newblock In \emph{STOC}, 2015.

\bibitem[M.Talagrand(2014)]{talagrand14}
M.Talagrand.
\newblock \emph{Upper and Lower Bound of Stochastic Process: Modern Methods and
  Classical Problems}.
\newblock Springer, 2014.

\bibitem[Paninski(2008)]{Paninski08}
L.~Paninski.
\newblock A coincidence-based test for uniformity given very sparsely sampled
  discrete data.
\newblock \emph{IEEE Trans. Inf. Theor.}, 54\penalty0 (10):\penalty0
  4750--4755, October 2008.
\newblock ISSN 0018-9448.
\newblock \doi{10.1109/TIT.2008.928987}.
\newblock URL \url{http://dx.doi.org/10.1109/TIT.2008.928987}.

\bibitem[Valiant and Valiant(2010{\natexlab{a}})]{cltVV}
Gregory Valiant and Paul Valiant.
\newblock A clt and tight lower bounds for estimating entropy.
\newblock In \emph{Electronic Colloquium on Computational Complexity (ECCC)},
  volume~17, page~9, 2010{\natexlab{a}}.

\bibitem[Valiant and Valiant(2010{\natexlab{b}})]{unseenVV}
Gregory Valiant and Paul Valiant.
\newblock Estimating the unseen: A sublinear-sample canonical estimator of
  distributions.
\newblock In \emph{Electronic Colloquium on Computational Complexity (ECCC)},
  volume~17, page~9, 2010{\natexlab{b}}.

\bibitem[Valiant and Valiant(2014)]{VV14}
Gregory Valiant and Paul Valiant.
\newblock An automatic inequality prover and instance optimal identity testing.
\newblock In \emph{FOCS}, 2014.

\bibitem[van Handel(2014)]{van2014probability}
Ramon van Handel.
\newblock Probability in high dimension.
\newblock Technical report, PRINCETON UNIV NJ, 2014.

\end{thebibliography}

\newpage

\appendix

\vskip 0.2in

\end{document}